\documentclass{article}

\usepackage{enumitem}

\usepackage{microtype}
\usepackage{graphicx}
\usepackage{subfigure}
\usepackage{booktabs} 

\usepackage{hyperref}

\usepackage[accepted]{icml2024}

\usepackage{amsmath}
\usepackage{amssymb}
\usepackage{mathtools}
\usepackage{amsthm}

\usepackage[capitalize,noabbrev]{cleveref}

\theoremstyle{plain}
\newtheorem{theorem}{Theorem}[section]

\newtheorem{lemma}[theorem]{Lemma}

\theoremstyle{definition}
\newtheorem{definition}[theorem]{Definition}

\theoremstyle{remark}

\usepackage[textsize=tiny]{todonotes}

\begin{document}

\twocolumn[
\icmltitle{No Free Prune: Information-Theoretic Barriers to Pruning at Initialization}




\icmlsetsymbol{equal}{*}

\begin{icmlauthorlist}
\icmlauthor{Tanishq Kumar}{equal,sch}
\icmlauthor{Kevin Luo}{equal,sch}
\icmlauthor{Mark Sellke}{sch}

\end{icmlauthorlist}

\icmlaffiliation{sch}{Harvard University}

\icmlcorrespondingauthor{Tanishq Kumar}{tkumar@college.harvard.edu}

\icmlkeywords{Machine Learning, ICML}

\vskip 0.3in
]



\printAffiliationsAndNotice{\icmlEqualContribution} 

\begin{abstract}
The existence of ``lottery tickets" \cite{frankle2018lottery} at or near initialization raises the tantalizing question of whether large models are necessary in deep learning, or whether sparse networks can be quickly identified and trained without ever training the dense models that contain them. However, efforts to find these sparse subnetworks without training the dense model (``pruning at initialization") have been broadly unsuccessful \cite{frankle2020pruning}. We put forward a theoretical explanation for this, based on the model's \emph{effective parameter count}, $p_{\text{eff}}$, given by the sum of the number of non-zero weights in the final network and the \emph{mutual information} between the sparsity mask and the data. We show the Law of Robustness of \cite{bubeck2021universal} extends to sparse networks with the usual parameter count replaced by $p_{\text{eff}}$, meaning a sparse neural network which robustly interpolates noisy data requires a heavily data-dependent mask. We posit that pruning during and after training outputs masks with higher mutual information than those produced by pruning at initialization. Thus two networks may have the same sparsities, but differ in \textit{effective} parameter count based on how they were trained. This suggests that pruning near initialization may be infeasible and explains why lottery tickets \textit{exist}, but cannot be found fast (i.e. without training the full network). Experiments on neural networks confirm that information gained during training may indeed affect model capacity.
\end{abstract}

\section{Introduction}
\label{submission}

Motivated by perennially growing model sizes and associated costs, neural network pruning is a technique used to reduce the size and cost of neural networks during training or inference, while maintaining performance on a task. 

Typically, pruning is done by masking away a certain fraction of weights (setting them to zero), so that they can be ignored for the purposes of training or inference, reducing the number of operations and thus cost required to achieve good performance on a task. There are three stages of the machine learning pipeline at which networks can be pruned:
\begin{enumerate}
    \item \emph{At initialization}, before training weights on the data \cite{lee2018snip, tanaka2020pruning, wang2020picking}, often using the network's connection structure or the loss landscape around initialization.
    \item \emph{During training}, usually in an gradual manner, starting with the dense network, training on the data, and pruning some fraction of the smallest magnitude weights, and repeating. Methods differ on what they do to weights after each prune step. The contribution of \cite{frankle2018lottery} was showing rewinding to initial weights is important for good performance at this stage of pruning. Other methods include \cite{h.2018to}, which iteratively prunes the smallest magnitude weights according to a predefined sparsity schedule. 
    
    \item \emph{After training}, before inference, usually using simple but effective heuristics involving dropping the lowest magntitude weights \cite{han2015learning}
\end{enumerate}

\cite{frankle2018lottery} show empirically that sparse subnetworks that can train to accuracy matching or exceeding that of the full, dense model, do indeed exist at or near initialization (``matching" subnetworks). This work and the follow-ups \cite{frankle2019stabilizing, frankle2020linear} present an algorithm to find these subnetworks called \emph{iterative magnitude pruning} (IMP) with weight rewinding, dubbing such subnetworks present at initialization ``lottery tickets." In principle, this raises the enticing prospect of quickly finding these networks at initialization and training only at high sparsity, but IMP requires repeatedly training the whole model on the dataset to find these lottery tickets, defeating the original point of finding highly sparse yet \textit{trainable} subnetworks.

Since then, research on ``pruning at initialization" has sought to find these lottery tickets fast (i.e., without training the full model).  Methods in stage (3) of the machine learning pipeline serve as benchmarks for sparsity, where those proposing pruning methods in stages (1) and (2) attempt to produce models as small as those that can be found in stage (3), where one can prune to reasonably high sparsities without compromising accuracy \cite{frankle2020pruning}. There remains to this day a large amount of effort to develop algorithms to prune at stage (1) \cite{tanaka2020pruning, lee2018snip, wang2020picking, pham2022understanding, wang2020picking}, but despite the diversity of techniques used, these algorithms are typically unsuccessful in finding  ``lottery tickets" in general settings without training. Note that our focus throughout this paper is restricted to algorithms producing \textit{matching} subnetworks, as it is of course possible to prune a network at initialization if one sacrifices accuracy.  

An important observation that if such an algorithm to find a matching subnetwork for a task at initialization did exist, it would suggest that the modern deep learning paradigm of training large models is misguided, as one could prune at initialization and then train the sparse subnetwork cheaply to achieve the same performance on a task. The existence of such an algorithm is also in tension with a large theoretical literature on the benefits and necessity of overparameterization \cite{allen2019convergence, du2018gradient, simon2023more, neyshabur2017implicit, geiger2019jamming, bartlett2020benign, montanari2024tractability}. Any theory that seeks to formalize the intractability of pruning at initialization, however, must also explain why lottery tickets can \textit{exist}, but not be found efficiently (i.e., without training the full network on the data). 

\section{Related Work}

\textbf{Lottery tickets and sparsity.} Pruning neural networks has a long history, from classic techniques that prune weights based on connectivity metrics involving the Jacobian and Hessian \cite{lecun1990optimal, hassibi1992second} to simple but effective modern methods based on weight magnitude \cite{han2015learning, wen2016learning, molchanov2016pruning} and more recently the Lottery Ticket Hypothesis \cite{frankle2018lottery} and associated follow-up works \cite{zhou2019deconstructing, chen2020lottery, frankle2020linear, gale2019state, liu2018rethinking}, see \cite{blalock2020state} for a general survey. \cite{paul2022unmasking} give a thorough loss landscape perspective on the lottery ticket hypothesis, but their work is empirical and not concerned with pruning at initialization, rather intending to illuminate the mechanism of IMP geometrically. 

One important limitation of lottery tickets is that they are are subnetworks with \textit{unstructured} sparsity (pruning individual weights), which is difficult to accelerate on modern hardware, whereas \textit{structured} sparsity (pruning entire neurons or convolutional channels) \cite{han2016eie, he2018amc, zhang2020sparch, mao2017exploring} is more exploitable by hardware and so often leads to more drastic performance gains despite lower levels of end-time sparsity. While we present our results to address the difficulty of finding lottery tickets, our theorems are based on notions of parameter counts which are easily adaptable to structured pruning.

\textbf{Pruning at initialization.} The most extensive empirical evaluation of pruning methods at initialization is \cite{frankle2020pruning}, which finds that the most popular methods for the task, including SNIP \cite{lee2018snip}, SynFlow \cite{tanaka2020pruning}, and GraSP \cite{wang2020picking}, barely outperform random pruning, and are significantly beaten by even naive methods to prune \textit{after} training \cite{han2015learning}. Their most relevant findings are: 
\begin{itemize}
    \item Allowing methods designed to prune at initialization to train the full network on more data or for longer improves the performance of the subnetwork derived from pruning in a smooth manner. Then, applying any of these pruning at initialization methods to a full network \textit{after training it} allows aggressive pruning without compromising accuracy. This suggests that something happens while training the full network which makes it possible to prune aggressively without sacrificing accuracy. 
    \item Methods using various statistics of the data at initialization like SNIP and GraSP (but not training) do no better than data-agnostic pruning at initialization (SynFlow), so that ``pruning at initialization" can roughly be seen as ``pruning data-agnostically."
    \item 
    \cite{frankle2020pruning} explicitly comment on how striking it is that methods that use such different signals (magnitudes; gradients; Hessian; data or lack thereof) end up reaching similar accuracy, behave similarly under ablations, and improve similarly when pruning after initialization. We argue there may be fundamental information-theoretic barriers causing these diverse methods to fail in very similar ways at high sparsities. 
\end{itemize}

\cite{evci2019difficulty} is another empirical work exploring the difficulty of training sparse networks, or, equivalently, pruning at initialization. Other works on efficiently finding lottery tickets include \cite{alizadeh2022prospect, jorge2021progressive}, with some new works such as \cite{ramanujan2020whats, sreenivasan2022rare} attempting to train a good mask at initialization by maximizing network accuracy. 



\textbf{Overparameterization and effective parameter count.} One of the biggest surprises in modern deep learning models is that overparameterized models generalize well and don't overfit \cite{zhang2021understanding}. A great deal of work has gone into quantitative analysis of this mystery. We build in particular on \cite{bubeck2021universal}, where it is shown that overparameterization by a factor of the data dimension is \textit{necessary} for \textit{smooth interpolation}. Other releveant work includes \cite{allen2019convergence}, which proves that SGD can find global minima in polynomial time of number of layers and parameters if sufficiently overparameterized, and \cite{du2018gradient} which proves gradient descent converges in linear time to a global optimum for a two-layer ReLU network if sufficiently overparameterized. \cite{simon2023more} prove that overparameterization is necessary for near-optimal performance in random feature regression. 

\textbf{Mutual information and generalization bounds.}
Since \cite{russo2019much} studied ``bad information usage" in the setting of adaptive data analysis, much work has been done on bounding the generalization gap using the informational quantity $I(W;\mathcal{D})$ where $W$ represents the chosen hypothesis by the learning algorithm, and $D$ the dataset sampled from a data distribution \cite{xu2017information, bu2020tightening, asadi2018chaining}. The main takeaway from these works is that learning algorithms whose mutual information with the data is low must generalize well. From a technical viewpoint our work is closely related to these, although the motivation is different.



\section{Contributions}

\begin{itemize}
    \item We state and prove a modified version of the Law of Robustness in \cite{bubeck2021universal} where parameter count is replaced with a data-dependent ``effective parameter count" that includes both the number of parameters and the mutual information of the sparsity mask with the dataset, showing a new way in which information and parameters can be traded off. In fact, this general principle of trading off parameter count for mutual information with the data extends beyond the Law of Robustness, as we show in Appendix \ref{app:gbounds}. 
    \item We examine the consequences of this result for the tractability of pruning at initialization, the most important of which is the observation that subnetworks derived from pruning algorithms that train on the data, such as lottery tickets, are not really sparse in \textit{effective parameter count}, whereas those derived from pruning at initialization are. We outline how this is may explain why lottery tickets exist, but cannot be found fast (i.e., without training the full network). 
    \item We perform experiments on neural networks where our mutual information quantities of interest can approximated and track these quantities during training. We find that at the same sparsity (parameter count), subnetworks derived from pruning algorithms that train the full network on the data have higher capacity and expressivity than those that prune at initialization, reflecting their higher \textit{effective} parameter count. 
\end{itemize}

\section{Informally Stated Results \& Implications}

Here we give an informal statement of our main theoretical result, followed by a discussion of its interpretation and consequences.
Formal statements can be found in Section 5, with full proofs deferred to the Appendix. 
We set the scene with the original Law of Robustness. 
\begin{theorem}[{Theorem 1, \cite{bubeck2021universal}}, informal]\label{thm:lor}
    Let $\mathcal{F}$ be a class of functions from $\mathbb{R}^d \rightarrow \mathbb{R}$ and let $\left(x_i, y_i\right)_{i \in[n]}$ be i.i.d. input-output pairs in $\mathbb{R}^d \times[-1,1]$. Assume that:
\begin{enumerate}[label=(\alph*)]
    \item $\mathcal{F}$ admits a Lipschitz parametrization by $p$ real parameters, each of size at most $\operatorname{poly}(n, d)$.
    \item The covariate distribution $\mu$ is a mixture of $O(n/\log n)$ ``truly high-dimensional'' components exhibiting certain concentration behavior. 
    \item The expected conditional variance of the output, $\sigma^2 \equiv \mathbb{E}^\mu[\operatorname{Var}[y \mid x]]>0$, is strictly positive.
\end{enumerate}
Then, with high probability over the sampling of the data, one has simultaneously for all $f \in \mathcal{F}$:
\[
\frac{1}{n} \sum_{i=1}^n\left(f\left(x_i\right)-y_i\right)^2 \leq \sigma^2-\epsilon \Rightarrow \operatorname{Lip}(f) \geq \widetilde{\Omega}\left(\frac{\epsilon}{\sigma} \sqrt{\frac{n d}{p}}\right).
\]
Here $\operatorname{Lip}(f)$ denotes the Lipschitz constant of $f$.
\end{theorem}
In essence, this theorem states that with high probability, any parameterized function $f \in \mathcal{F}$ that has training error below the noise level $\sigma^2$ has smoothness $1/\text{Lip}(f)$ increasing in the number of parameters. Thus overparameterization is, under these assumptions, \textit{necessary} for smooth interpolation, presumably a pre-requisite for robust generalization (more in Section~\ref{sec:prelim}). 

Our main result combines this with information theoretic results, giving in a modified version of the Law of Robustness that suggest fundamental limits for pruning at initialization.

\begin{theorem}[Informal, Modified Law of Robustness]\label{thm:informal}
Assume the same conditions as in Theorem~\ref{thm:lor} and that $\mathcal{F}$ has the additional structure of masks, so that each hypothesis $f \in \mathcal{F}$ has parameters $(\mathbf{m}, \mathbf{w})$ satisfying $\mathbf{m}_i = 0 \implies \mathbf{w}_i = 0$ for all $i \in [p]$. Then, with high probability over sampling of the data, one has for any learning algorithm $W$ taking in data $\mathcal{D}$ and outputting function $f^W \in \mathcal{F}$:
\begin{multline}
    \frac{1}{n} \sum_{i = 1} ^n (f^W(x_i) -y_i)^2 \leq \sigma^2 - \epsilon \\
    \implies \mathrm{Lip}(f) \geq \widetilde\Omega\left(\epsilon \sqrt{\frac{nd}{p_{\text{eff}}}}\right),
\end{multline}
where $p_{\text{eff}} = \tilde{\Theta}\left(I(\mathbf{m}^W; \mathcal{D}) + \mathbb{E}[\|\mathbf{m}\|_1]\right)$. 
\end{theorem}
The above theorem replaces the parameter count $p$ with an effective parameter count $p_{\text{eff}}$. As one would expect, $p_\text{eff}$ can never be larger than $p$\footnote{Up to logarithmic factors incurred from discretization.}. This effective parameter count includes the number of unmasked parameters $\mathbb{E}[\|\mathbf{m}\|_1]$ as well as the mutual information between the sparsity pattern $\mathbf{m}^W$ and the data. This latter term captures the intuition that if the sparsity mask for a given network is learned from the data, as it is in contexts such as IMP, it should be interpreted as a set of binary parameters and hence contribute to the overall parameter count of the network. The implications of the above are that pruning algorithms that iteratively use properties of the data to find a mask may not be truly sparse in this effective parameter count, as we will see verified in our experiments. They trade off the unmasked parameter count of the network for mutual information, so that the subnetwork produced has effective parameter count much larger than a truly sparse network. Conversely, a learning algorithm pruning at initialization with little to no dependence on data will result in a subnetwork that has low effective parameter count, and thus poor robustness, since it is truly sparse. 

We give an example which saturates the bound presented above in Appendix~\ref{app:tight}. 

\color{black}
\textbf{Overparameterization and Mutual Information.}  Following \cite{bubeck2021universal}, we show that for the learned function to fit below the noise level, it must correlate with the noise in the data. The key difference in our proof is a result from information theory (Lemma~\ref{lem:mi-bound} below) which bounds this correlation by the mutual information between the learned function and the data. 
This bound has been used in the literature to show generalization error is \textit{small} with high probability when the learned function has low mutual information with the data. Instead, we show that since the function fits below the noise level, this correlation must be \textit{large}, and thus the mutual information must be large with high probability. Although classical generalization bounds recommend a low parameter count, from the perspective of the Law of Robustness (and much of modern deep learning practice), overparameterization is \textit{necessary} to find a hypothesis interpolating the data \textit{smoothly}. As a result, the chosen hypothesis must have high mutual information with the data for this to be possible (though it is not sufficient). 

\textbf{On Fitting Below Noise Level.} Our results are restricted to the regime where the fitted function $f$ fits below the noise level $\sigma^2$ of the data, but this is a weak assumption usually satisfied in practice. For example, in computer vision, the setting in which pruning was first investigated, it is common to see networks with near-perfect test accuracy (so they have learned all the relevant signal) have end-time train loss lower than end-time test loss (so they must have memorized some noise). A typical example where such train-test loss behavior can be seen in a practical setting is in the ResNet paper \cite{he2016deep}, and other examples include \cite{huang2017densely, lin2013network}. Further, \cite{feldmanlongtail} suggests that this phenomenon of fitting below the noise level may be \textit{necessary} for achieving optimal generalization error on real-world datasets. This is a consequence of such data often being a mixture of subpopulations where the distribution of subpopulation frequencies often follows power laws and is thus long-tailed. One can also interpret $\sigma^2$ as the portion of a given task that is ``hard to learn" for the choice of model - that is, for a given model class, it can be seen as the residual variance conditional on a ``good" feature representation, which will be larger due to portions of the target function falling outside this representation. This behavior is by now well established in several theoretical settings, see e.g. \cite{jacot2020implicit,Canatar_2021,ghorbani2021linearized,misiakiewicz2023six}. We defer to \cite{bubeck2021universal} for further discussion. 

\subsection{Interpretation of $p_{\text{eff}}$}\label{sec:peff}

We now discuss the intuition behind the parameter $p_{\text{eff}}$ in Theorem~\ref{thm:informal}. In practice, one generally attempts to prune a model to a specified sparsity level $\gamma$, so that $(1 - \gamma)p$ parameters remain, with $1 - \gamma$ taken to be small, often around 5\% or less. Thus $\mathbb{E}[\|\mathbf{m}^W\|] = (1 - \gamma)p$. We then compare $p_{\text{eff}}$ for the following two pruning schemes:

\begin{itemize}
    \item Pruning at initialization in a data agnostic manner (e.g. SynFlow \cite{tanaka2020pruning}). The effective parameterization is $\tilde{\Theta}((1 - \gamma) p)$, reflecting that one is simply training a model of much smaller size, and hence does not obtain any benefits from the overparameterization of the initial dense model\footnote{Here $\tilde\Theta$ hides logarithmic terms in weight size and dependence on $\delta$.}. 
    \item Pruning with IMP, which we argue has high mutual information. For the purposes of illustration, suppose that ``high mutual information" means being on the order of its upper bound $H(\mathbf{m}^W) \leq \log_2 \binom{p}{\gamma p}$. For large $p$ and fixed $\gamma\in (0,1)$,
    \begin{align}\label{eq:entropy-max}
        &I(\mathbf{m}^W; \mathcal{D}) \simeq \log_2 \binom{p}{\gamma p} \\
        \notag
        &= \left(p(1 - \gamma) \log_2 \tfrac{1}{1 - \gamma} + \gamma \log_2 \tfrac{1}{\gamma}\right)-o(p).
    \end{align}
    The last bound follows from \cite{csiszar2004information} Lemma 2.2, and the total effective parameter count is this quantity combined with $\tilde\Theta((1 - \gamma) p)$. 
    We note that the extreme case of \eqref{eq:entropy-max} corresponds to Eq. (2.13) in \cite{bubeck2021universal}, i.e. a worst-case bound on the effective parameter count.
\end{itemize}

The ratio in the effective parameter count between the second setting and the first then scales as $\tilde\Theta\left(\log_2 \frac{1}{1 - \gamma}\right)$
which diverges as $\gamma \to 1$, illustrating that the second setting has increasingly more times as many effective parameters as the first. This difference reflects a possible barrier between IMP and methods which prune at initialization. While IMP can indeed find sparse subnetworks at initialization, these sparsity patterns are found after the model has been trained on the entire dataset. Hence, we believe that such masks have high mutual information with the data, and it is because of this that the subnetworks chosen by IMP perform better, even at high sparsities. By contrast methods which prune at initialization (either data agnostically or using the loss landscape around initialization) ought to produce masks which have no or very little mutual information with the data. Our results suggest the trained networks produced in the latter cases cannot generalize well at very high sparsity.

\section{Theoretical results}


\subsection{Preliminaries}\label{sec:prelim}

\textit{Isoperimetry}
Our results, which follow those of \cite{bubeck2021universal}, assume that the distribution of the data covariates $x_i$ are \textit{isoperimetric} in the following sense:
\begin{definition}\label{def:isoperimetry}
   A probability measure $\mu$ on $\mathbb{R}^d$ satisfies $c$-isoperimetry if the following holds. For any bounded $L$-Lipschitz $f: \mathbb{R}^d \to \mathbb{R}$, 
   \begin{equation}
       \mathbb{P}\left(|f(x) - \mathbb{E}[f]| \geq t\right) \leq 2e^{-\frac{dt^2}{2cL^2}},\quad\forall t\geq 0.
   \end{equation}
\end{definition}
Isoperimetry asserts that Lipschitz functions concentrate sharply around their mean, and is a ubiquitous property of ``truly high-dimensional'' distributions such as Gaussians. 

\textit{Information Theoretic Concentration Bounds}



We make use of the following Lemma in our proofs. 
\begin{lemma}[{\cite{xu2017information}}]\label{lem:mi-bound}
    Let $\{X_t\}_{t \in T}$ be a random process and $T$ an arbitrary set. Assume that $X_t$ is $C$-subgaussian\footnote{We use the convention that an $\mathbb{R}$-valued random variable $X$ is $C$-subgaussian when $X$ satisfies $\mathbb{P}(|X| \geq t) \leq 2e^{-t^2/C}$.} and $\mathbb{E}[X_t] = 0$ for every $t \in T$, and let $W$ be a random variable taking values on $T$. Then for some absolute constant $a_1 > 0$, with probability $1 - \delta$,
    \begin{equation}\label{eq:mi-bound}
        |X_W| < a_1 \sqrt{(C/\delta)\cdot
        I(\{X_t\}_{t \in T}; W)+ C\log(2/\delta)}.
    \end{equation}
\end{lemma}

To understand this result, it is helpful to consider two extremes. First if $W$ is independent from $\{X_t\}_{t \in T}$, then standard concentration results show \eqref{eq:mi-bound} holds with extremely high probability without the mutual information term, so the size of $T$ is irrelevant.
On the other hand, in the worst case $X_W$ could be the largest of $T$ subgaussian variables, causing it to scale with $\sqrt{\log |T|}$. The above result allows us to interpolate between these two regimes. 
We note that the overall strategy of our argument is technically very similar to that of \cite{xu2017information} and related work such as \cite{russo2019much}. The conceptual difference is that we measure the information in the \emph{choice of function class} rather than on the level of individual functions.


\subsection{Notation}
The $p$-parameter function class $\mathcal{F}$ is indexed by a set $T$, such that $\mathcal{F} = \{f^t\}_{t \in T}$. 
Each function $f^t \in \mathcal{F}$ is parameterized by $(\mathbf{m}^t, \mathbf{w}^t)$, where $\mathbf{m}^t \in \{0, 1\}^p$ is the mask and $\mathbf{w}^t \in \mathbb{R}^p$ is the weightings on the unmasked parameters. As such, one has $\mathbf{m}^t_i = 0 \implies \mathbf{w}^t_i = 0$. We assume that each of the $(\mathbf{m}^t, \mathbf{w}^t)$ pairs are unique - that is, we assume the functions have unique parameters, though they need not correspond to unique functions. We then represent the learning algorithm as a random variable $W$ taking values on the index set $T$, which can depend on the data $\mathcal{D}$. That is, the function outputted by the learning algorithm is $f^W$.  For all of our results, we require that the data $\mathcal{D}$ takes the form $\{(x_i, y_i)\}_{i \in [n]}$, where the $(x_i, y_i)$ are i.i.d input-output pairs in $\mathbb{R}^d \times [-1, 1]$ satisfying the following two assumptions:
\begin{enumerate}
    \item[A1.] The distribution $\mu$ of the covariates $x_i$ can be written as $\mu = \sum_{\ell = 1} ^k \alpha_\ell \mu_\ell$, where each $\mu_\ell$ satisfies $c$-isoperimetry and $\alpha_\ell \geq 0$, with $\sum_{\ell = 1} ^k \alpha_\ell = 1$.
    \item[A2.] The average conditional variance of the output is strictly positive: $\sigma^2 \equiv \mathbb{E}^{x\sim\mu}[\mathrm{Var}[y \mid x]]>0$ . 
\end{enumerate}

Lastly, for each datapoint $i \in [n]$, denote the mixture component from which it is drawn as $\ell_i \in [k]$. 

\subsection{Finite Setting}
As in \cite{bubeck2021universal}, we begin in setting where $T$ (and thus $|\mathcal{F}|$) is finite. 
For each $\mathbf{m} \in \{0, 1\}^p$, define $\mathcal{W}_{\mathbf{m}} = \{\mathbf{w} \mid \exists f^i = (\mathbf{m}^i, \mathbf{w}^i) \in \mathcal{F} \text{ s.t. }\mathbf{m}^i = \mathbf{m}\}$, the possible weightings of the network once a mask is fixed, and additionally let $N_{\mathbf{m}} = |\mathcal{W}_{\mathbf{m}}|$. In this section, we denote $g(x) = \mathbb{E}[y \mid x]$ as the target function to learn and $z_i = y_i - g(x_i)$ as the noise.  We then obtain the following result: 
\begin{theorem}\label{thm:finite}
    If all $f \in \mathcal{F}$ have Lipschitz constant bounded above by $L$, then 
    \begin{multline*}
        \mathbb{P}\left(\frac{1}{n} \sum_{i = 1} ^n (y_i - f^W(x_i))^2 \leq \sigma^2 - \epsilon \right) \leq 
        (2k + 2) e^{-\frac{n\epsilon^2}{8^3 k}}
        \\
        + \max\left(\frac{ 2^{7} a_1^2 C_0 p_{\text{eff}}}{\epsilon^2}, 
        2e^{-\epsilon^2/(2^7 a_1^2 C_0)}
        \right),
    \end{multline*}
    where $p_{\text{eff}} = I(\mathbf{m}^W ; \mathcal{D}) + \mathbb{E}\left[\log_2 N_{\mathbf{m}^W}\right]$ and $C_0 = \frac{144cL^2}{nd}$.
\end{theorem}
This first theorem illustrates that the quantity of interest that dictates the size of the model class is now $p_{\text{eff}}$, which we refer to as the effective parameter count. 
We sketch of the proof Theorem~\ref{thm:finite} below. We begin with the following lemma:
\begin{lemma}\label{lem:noise-to-corr}
    \begin{align}
        &\mathbb{P}\left( \frac{1}{n} \sum_{i = 1} ^n (y_i - f^W(x_i))^2 \leq \sigma^2 - \epsilon \right) \leq (2k + 2)e^{-\frac{n\epsilon^2}{8^3 k}}
        \notag \\
        &\qquad + \mathbb{P}\left(  \frac{1}{n} \sum_{i = 1} ^n (f^W(x_i) - \mathbb{E}^{\mu_{\ell_i}}[f^W(x_i) | W])z_i \geq \frac{\epsilon}{8} \right) \notag
    \end{align}
\end{lemma}

In essence, for the chosen function $f^W$ to fit below the noise level, it must have large correlation with the noise $z_i$. It then remains to bound the probability of correlating with the noise. Define $X_t = \frac{1}{n} \sum_{i = 1} ^n (f^t(x_i) - \mathbb{E}^{\mu_{\ell_i}}[f^t])z_i$. That this quantity for a fixed function $f^t$ is subgaussian is a consequence of isoperimetry - one observes from Definition~\ref{def:isoperimetry} that an appropriate scaling of $(f^t(x_i) - \mathbb{E}^{\mu_{\ell_i}}[f^t])$ must be $O(1)$-subgaussian. The new element of our proof is to use Lemma~\ref{lem:mi-bound} to control this term rather than using standard tail bounds coupled with a union bound over all of $\mathcal{F}$. Explicitly, we apply the lemma taking $X_t$ to be this correlation term, which must be large for the function to fit below the noise level, and hence its upper bound must be large, rather than taking it to be a generalization error (which one usually hopes to be small).

We then control $I(W; \{X_t\}_{t \in T})$ as follows:
\begin{lemma}\label{lem:mi-control}
    \begin{multline}
        I(W; \{X_t\}_{t \in T}) \leq I(W; \mathcal{D}) \\
        = I(f^W; \mathcal{D}) \leq I(\mathbf{m}^W ; \mathcal{D}) + \mathbb{E}\left[\log_2 N_{\mathbf{m}^W}\right]
    \end{multline}
\end{lemma}

Note that the final expression is exactly our effective parameter count $p_{\text{eff}}$. Combining these two lemmas and Lemma~\ref{lem:mi-bound} obtains Theorem~\ref{thm:finite} (see Appendix~\ref{app:proofs}).

\subsection{Main Result}
We obtain the following analog of the Law of Robustness for continuously parametrized function classes:
\begin{theorem}\label{thm:contin}
    Let $\mathcal{F}$ be a class of functions from $\mathbb{R}^d \to \mathbb{R}$ and let $(x_i, y_i)_{i \in [n]}$ be input-output pairs in $\mathbb{R}^d \times [-1, 1]$.  Fix $(\epsilon, \delta) \in (0, 1)$. Assume that
    \begin{itemize}
        \item The function class can be written as $\mathcal{F} = \{f_{\mathbf{m}, \mathbf{w}} \mid \mathbf{m} \in \{0, 1\}^p, \mathbf{w} \in \mathcal{W}_{\mathbf{m}} \subset \mathcal{W}\}$ with $\mathcal{W} \subset \mathbb{R}^p$, $\mathrm{diam}(\mathcal{W}) \leq W$ and $(\mathbf{m}, \mathbf{w})$ satisfying $\mathbf{m}_i = 0 \implies \mathbf{w}_i = 0$ for all $i \in [0, p]$. 
        Furthermore, for any $\mathbf{w}_1, \mathbf{w}_2 \in \mathcal{W}$, 
        \begin{equation} \label{eq:param-const}
            \|f_{\mathbf{w}_1} - f_{\mathbf{w}_2}\|_\infty \leq J\|\mathbf{w}_1 - \mathbf{w}_2\|.
        \end{equation}
        \item Assumptions A1, A2 hold, with $8^3k \log(8k/\delta) \leq n\epsilon^2$.

        Then one has that for the learning algorithm $f^W$ taking values in $\mathcal{F}$, with probability at least $1 - \delta$ with respect to the sampling of the data, 
        \begin{multline*}
            \frac{1}{n} \sum_{i = 1} ^n (f^W(x_i) -y_i)^2 \leq \sigma^2 - \epsilon \\
            \implies \mathrm{Lip}(f) \geq \frac{\epsilon}{96a_1 \sqrt{2c}} \sqrt{\frac{nd\delta}{p_{\text{eff}} + \frac{\delta}{2} \log (4/\delta) }}
        \end{multline*}
        where 
        \[
            p_{\text{eff}} = I(\mathbf{m}^W; \mathcal{D}) + \mathbb{E}[\|\mathbf{m}\|_1] \log_2 (1 + 60WJ\epsilon^{-1}).
        \]
    \end{itemize}
    \end{theorem}

The proof of the above follows from Theorem~\ref{thm:finite}, coupled with a discretization argument and some careful modifications of the learning algorithm $W$. One should regard $\delta$ and $\epsilon$ as small, fixed constants independent of $n$ and $d$, at which point the theorem suggests that in order $f^W$ to be smooth, one must have $p_{\text{eff}} = \Omega\left(nd\right)$, recovering a result analogous to \cite{bubeck2021universal}. We remark that under standard training procedures, the above holds conditional on the specific weight initialization of the network, i.e. with $I(\mathbf{m}^W; \mathcal{D})$ replaced with $I(\mathbf{m}^W; \mathcal{D} \mid \text{init} = \text{init}_0)$, since the initialization is sampled independently before training. 


\subsection{Discussion}



\textbf{Brute Force Search.} Why cannot we brute force search over exponentially many subnetworks, testing each on the dataset, and claiming to have found the lottery ticket without training a dense model by the end of this process? The answer to this is that such a search creates high \textit{mutual information} between the mask and the data, since it involves testing every masked network on the data and choosing the best performing mask. Thus the mutual information based bounds we present in this work hold, because the effective parameter count of the mask derived from this brute force search would be high. Of course, training the full model is a faster way to find a lottery ticket than a brute force search, since training neural networks takes subexponential time \cite{livni2014computational}. 
Algorithms attempting to learn the mask at initialization \cite{sreenivasan2022rare, ramanujan2020whats}, before learning weights, are interesting in that this naturally increases $I(\mathbf{m}^W; \mathcal{D})$ early on in training, although their results are not demonstrated at ImageNet scale, where lots of lottery ticket type results are known to break down \cite{gale2019state, frankle2019stabilizing}. 

\textbf{Necessary conditions.} Our results present necessary but not sufficient conditions for smooth interpolation with a parameterized function class. \cite{bombari2023beyond}  examines whether the overparameterization condition $p = \Omega(nd)$ is \textit{sufficient} to guarantee smooth interpolation in various neural network regimes, finding that it is not for a random feature network, but \textit{is} for a network in the NTK regime of neural network training \cite{jacot2018neural}.\footnote{\cite{bombari2023beyond} do not include regularization in their random feature model, whereas \cite{simon2023more} assume optimal regularization, explaining the seemingly different conclusions about overparameterized RF regression models.} 

\textbf{Subtleties Around Effective Parameters.} While our idea of effective parameters suggest that one can trade off $I(\mathbf{m}^W;\mathcal D)$ and parameter count, this does not imply that a learning algorithm that artificially inflates this mutual information, for instance by encoding a discretization of the data within the mask, will do well as a pruning algorithm. Our modified Law of Robustness presents a \textit{lower bound} on Lipschitz constant, so that the actual Lipschitz constant could easily be higher than the lower bound given. Additionally, one can view the choice of a mask $\mathbf{m}^W$ followed by a choice of weights $\mathbf{w}^W$ constrained by this mask as having access to a collection of function classes $\{\mathcal H_t\}_{t \in T}$, where one first makes a choice $W$ (possibly depending on the data $\mathcal{D}$) of function class $\mathcal{H}_W$ and then a choice of function within the class. In such cases, this idea of effective parameterization accounting for both the mutual information of the choice of function class and the parameter count within the chosen class extends beyond the Law of Robustness to other settings (see Theorems~\ref{thm:vc-mi}, \ref{thm:rc-mi} in Appendix~\ref{app:gbounds}), showing that the principle we introduce of trading off parameters and information is a general one. 


\label{sec:limitations}
\textbf{Limitations.} While pruning was historically grounded in the computer vision setting, where large models make many passes over the data during training and thus memorize noise, recent large language models (LLMs) are trained in the online regime, and such models do not interpolate because each data point is only seen once during training. While we suspect results analogous to ours hold for LLMs, our condition of fitting below the noise
does not apply to such settings, so different theoretical approaches will be needed, and this is an important avenue for future work. Our results are non-constructive in that they show hard bounds on the limits of pruning to find generalizing subnetworks as a function of the information masks have with the dataset, but since our proof applies to a broad range of parameterized function classes (not just neural networks), we do not construct an algorithm for optimal pruning (in the information-theoretic sense) as we do not quantify the optimal tradeoff between mutual information and parameter count in terms of pruning. Doing so to yield specific algorithms at the edge of optimality is left for future work. 

\section{Experiments}
\begin{figure*}[ht!]
    \centering
    \subfigure[
    ]{\label{fig:memorization-left}
    \includegraphics[width=0.32\linewidth]{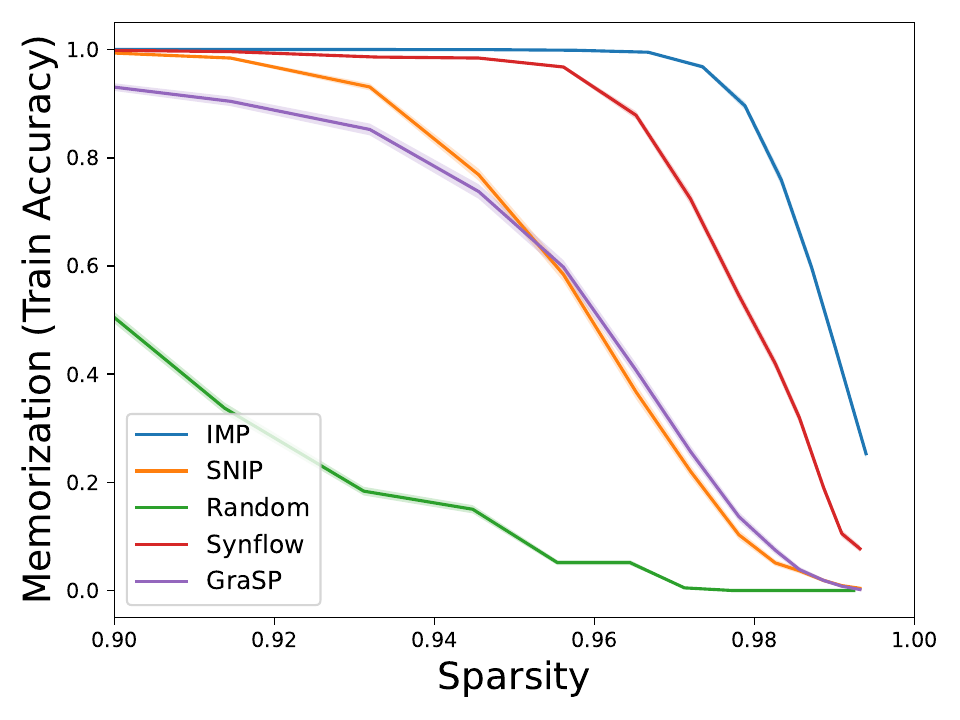}}
    \subfigure[
    ]{\label{fig:memorization-right}
    \includegraphics[width=0.32\linewidth]{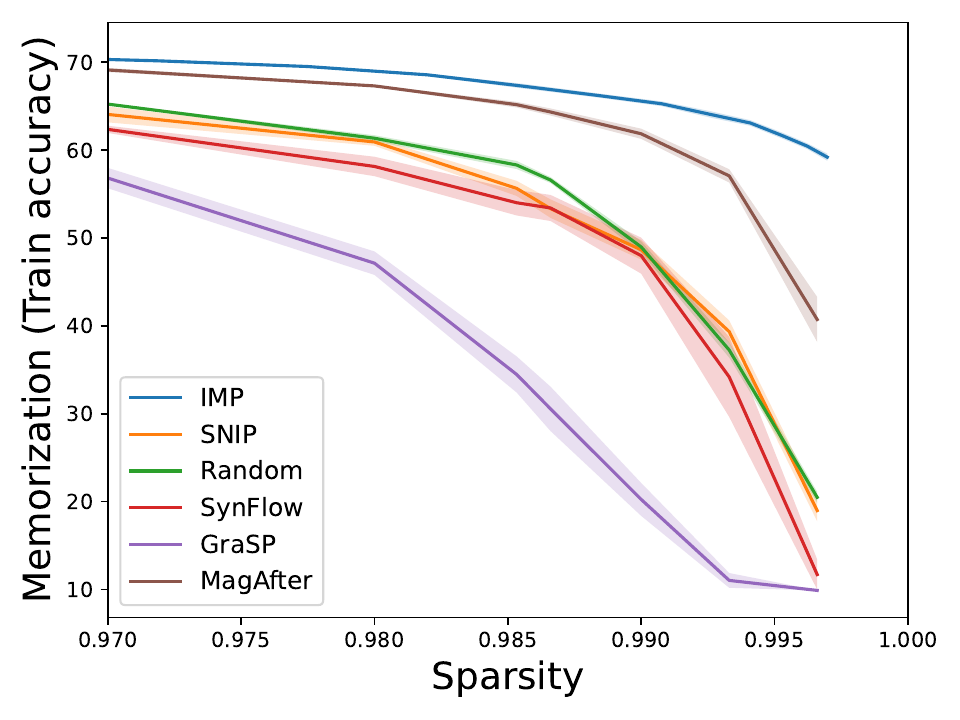}}
    \subfigure[]{\label{fig:cifar-final}
    \includegraphics[width=0.32\linewidth]{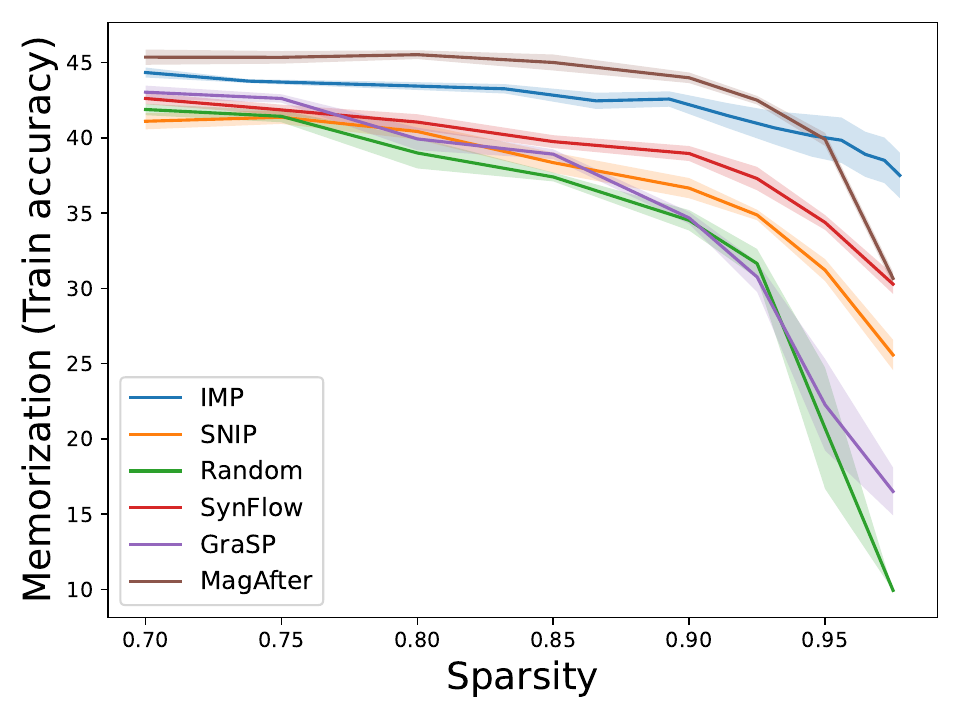}}
    \subfigure[]{\label{fig:imp_corr_epochs_small}
    \hfill\includegraphics[width=0.32\linewidth]{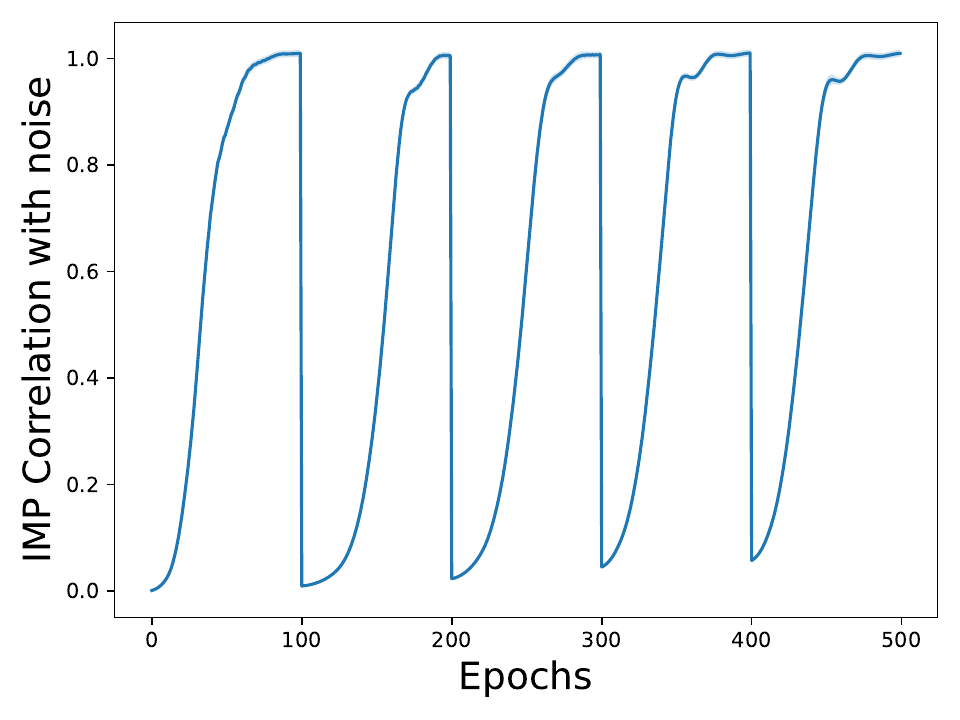}}  
    \subfigure[]{\label{fig:corr-middle}\includegraphics[width=0.32\linewidth]{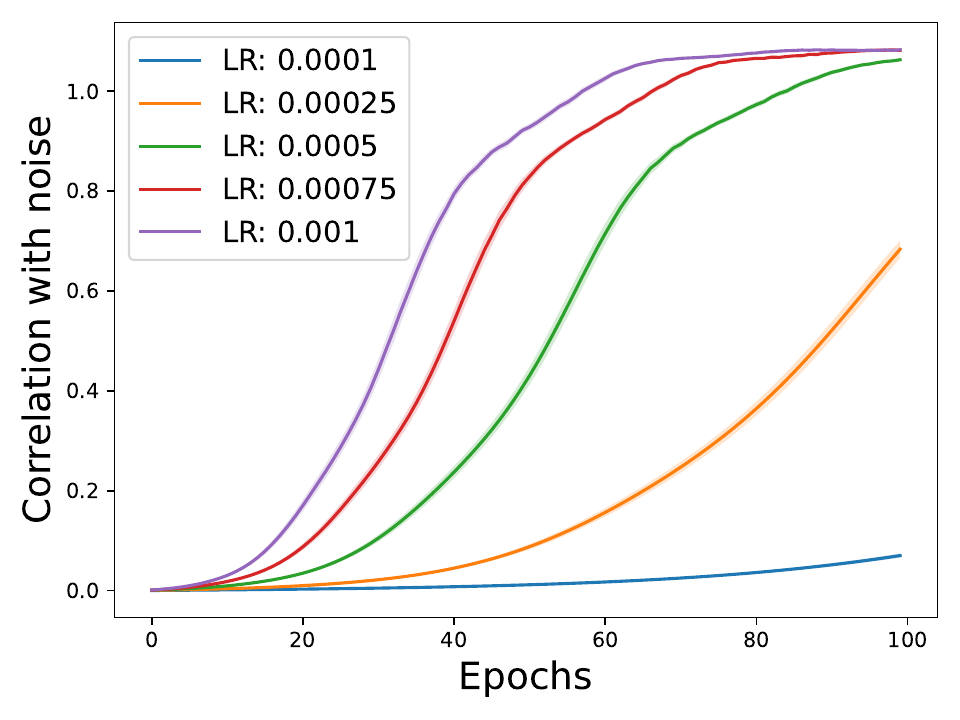}}
    \subfigure[]{\label{fig:corr-right}\includegraphics[width=0.32\linewidth]{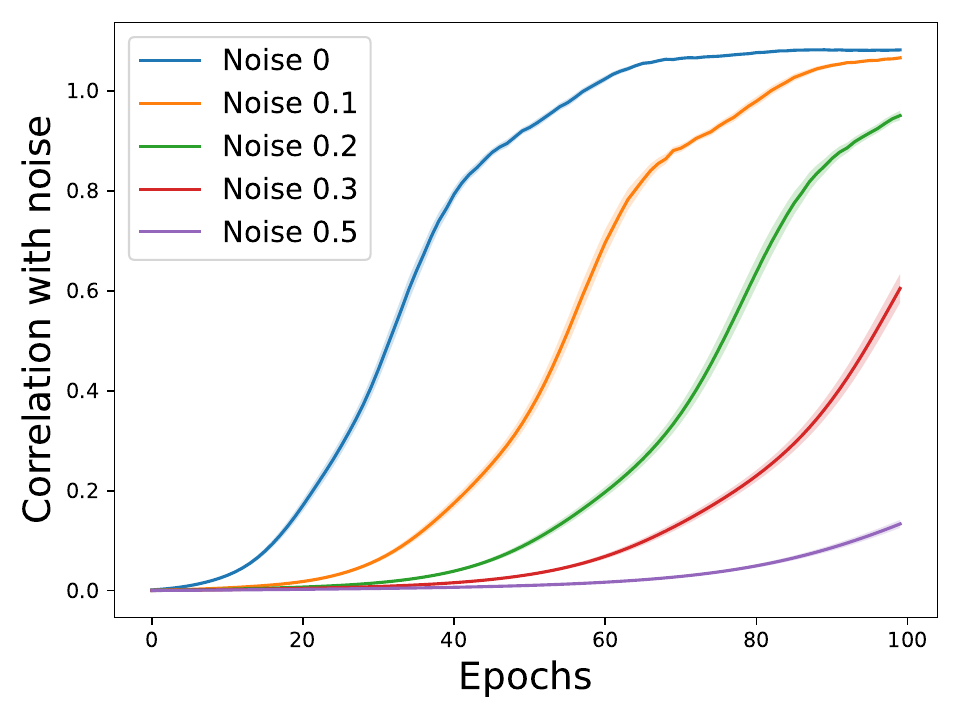}}
    \caption{\textbf{Top}: memorization capacity (train accuracy on noisy data) against sparsity level for different pruning methods. Staying higher on plots is better. Vertical gap between IMP/Magnitude-after pruning reflect additional memorization capacity on this dataset due to mutual information between mask and data. 2-hidden layer MLP on Gaussian data \textbf{\subref{fig:memorization-left}} and noisy FashionMNIST \textbf{\subref{fig:memorization-right}}, 4-layer convNet on noisy CIFAR-10 \textbf{\subref{fig:cifar-final}}. \textbf{Bottom}: Ability to correlate with dataset noise over training as a proxy for network capacity. 5-layer ReLU MLP in a student-teacher task, with $\sigma^2 = 1$ noisy labels; \textbf{\subref{fig:imp_corr_epochs_small}}: correlation with noise during IMP; training increases correlation with noise (hence $p_{\text{eff}}$), pruning then reduces this, before repeating; \textbf{\subref{fig:corr-middle}}: Sweep over learning rates; \textbf{\subref{fig:corr-right}}: Sweep over amount of noise injected into gradients. This illustrates that our quantity of interest, $I(f^W;\mathcal{D})$, and thus $p_{\text{eff}}$, is increasing due to data contained in gradients.}
\end{figure*}


\textbf{Motivating the Experimental Setup.} The main message of our theory is that pruning at initialization may be inherently difficult because a mask derived through training a dense network contributes to effective parameter count. How can one test whether this occurs in neural networks in practice? Ideally, one would track the mutual information of interest, $I(\mathbf{m^W}; \mathcal{D})$ throughout training, but because computing mutual information requires estimating high dimensional distributions, it is only exactly feasible in very small models \cite{kraskov2004estimating, gao2018demystifying}, which is precisely the setting where pruning to high sparsities leads to layer collapse \cite{tanaka2020pruning}, where all the weights in one layer vanish. We perform such experiments for a tiny neural network in Appendix~\ref{app:toymodel}. 
The question then becomes how else one quantify the ``effective parameter count" of a network. We resort to tracking two proxies for mutual information and effective parameter count: memorizing capacity and correlation with noise. The first asserts that ``effective parameter count" can be measured by tracking ``memorization capacity" \cite{cover1965geometrical, sur2019modern, montanari2024tractability}. And if one wants to understand ``ability to memorize," using train accuracy noisy data is a natural candidate. Our second proxy for mutual information is more direct: the correlation of the network function with the noise in labels over the course of training. By Lemma \ref{lem:mi-bound}, we have that this correlation is a lower bound for the quantity of interest $I(f^W; \mathcal{D})$.

To reiterate: what our theory seeks to explain is the empirical finding in  \cite{frankle2020pruning} that shows that methods that derive a data-dependent mask (lottery tickets) outperform methods that do not in terms of end-time test accuracy. So any quantity one argues causes this trend in end-time test accuracy (in our case, mutual information and thus effective parameter count) must exhibit the same trend. Our theory predicts mutual information and thus effective parameter count of the sparse network after pruning should be highest for IMP and magnitude-after pruning, and lower for methods that prune at initialization, like SNIP, GraSP, and SynFlow. Our experiments are meant to illustrate that information gained during training can affect model capacity, rather than serve as exhaustive examinations of large modern architectures.

\subsection{Memorization capacity}


We consider a two-hidden layer network with ReLU activation, on a train set of points (Gaussian data in Figure \ref{fig:memorization-left} and FashionMNIST in Figure \ref{fig:memorization-right}) with noise (random Boolean labels in Figure \ref{fig:memorization-left}, and corrupted images with true labels in Figure \ref{fig:memorization-right}), as well as a convNet on noisy CIFAR-10 in \ref{fig:cifar-final}, and compute the fraction of the training set the sparse network can memorize at each sparsity level. 

\color{black}

One sees that for all datasets here IMP is more expressive and can memorize more points \textit{with the same number of unmasked parameters} as the other subnetworks derived from different pruning algorithms. This gap in memorization capacity is due to mutual information with the data, providing evidence that the effective parameter count our theorems reason can indeed correspond to capacity of  neural networks in practice. We emphasize that the y-axis is \textit{not} test accuracy, but \textit{memorization capacity} (train accuracy on highly noisy data): the fact that our plots give analogous behavior to plots of \textit{noiseless test-accuracy against sparsity} in, e.g., \cite{frankle2020pruning}, is encouraging because the shape of such plots is \textit{what we sought to explain}. Experimental details are in Appendix \ref{app:expdeets}.

\subsection{Correlation with noise} 
We now track our second proxy for mutual information/effective parameter count during training, correlation with label noise. Overall, we see that as the network reaches low train error, it necessarily fits some of the noise, and so correlation with noise increases. We measure this quantity over training in a student-teacher task. Figure \ref{fig:corr-middle} illustrates how increasing learning rate increases the rate at which MI is acquired, suggesting that gradients carry this informational quantity. To verify this, we can corrupt the gradients with noise, finding in Figure \ref{fig:corr-right} that this correlational proxy for MI vanishes when we do so. Finally, Figure \ref{fig:imp_corr_epochs_small} tracks this proxy for MI during IMP, finding that the MI increases during iterations of pruning, decreases after pruning, and breaks down at late time. In context of our theory where $I(\mathbf{m^W};\mathcal{D})$ contributes to $p_{\text{eff}}$, this shows the role of \textit{iterative}, rather than one-shot pruning, may be to increase the effective parameter count via gradient updates, so it can be pruned back down again. Hence the \textit{iterative} nature of pruning allows the trade-off between parameters and information to be made many times. 
    
    \textbf{Takeaways.} Our experiments illustrate how mutual information can be traded off with parameters so that subnetworks derived from pruning at initialization have a lower effective parameter count than those derived from pruning after training. Our work explains why widespread empirical efforts to prune at initialization have run into difficulties, and introduces a new information-parameter trade-off in doing so that may be of independent interest in understanding neural networks trained by gradient descent. 

\newpage 
\section*{Impact Statement}

Our theoretical results are suggestive about the difficulty of pruning neural networks at initialization. We find that there may be information-theoretic barriers to finding extremely sparse, trainable subnetworks without training the full neural network on data so that research efforts to prune to high sparsities may instead be better directed at pruning jointly with training or before inference, where methods are known to exist for pruning to high sparsities without compromising accuracy.  This contributes to a wider literature on the limitations of neural networks.

\bibliography{pruning_bib}
\bibliographystyle{icml2024}

\newpage
\appendix
\onecolumn

\section{Proof of Theorem~\ref{thm:contin}}\label{app:proofs}

Without loss of generality, we will assume that all functions in $\mathcal{F}$ have range contained in $[-1, 1]$.
This is possible because clipping larger outputs to the closest point in $[-1,1]$ can only improve the Lipschitz constant and mean squared error. 



\subsection{Overview of Finite setting, Theorem~\ref{thm:finite}}

\subsubsection{Proof of Lemma~\ref{lem:noise-to-corr}}

\begin{lemma}\label{lem:noise-to-raw-corr}
    Denote the events
    \begin{align*}
        A_f = \left\{\frac{1}{n} \sum_{i = 1} ^n (y_i - f(x_i))^2 \leq \sigma^2 - \epsilon\right\} \qquad B &= \left\{ \frac{1}{n} \sum_{i = 1} ^n z_i^2 \geq \sigma^2 - \frac{\epsilon}{6}\right\} \qquad C = \left\{ \frac{1}{n} \sum_{i = 1} ^n z_i g(x_i) \geq -\frac{\epsilon}{6} \right\}  \\
        D_f &= \left\{ \frac{1}{n} \sum_{i = 1} ^n f(x_i) z_i \geq \frac{\epsilon}{4} \right\}.
    \end{align*}
    Then $A_f \cap B \cap C \implies D_f$, and $\mathbb{P}(B^c \cup C^c) \leq 2 \exp \left( -\frac{n\epsilon^2}{8^3} \right)$.
\end{lemma}
\begin{proof}
    Follows from Lemma~2.1 of \cite{bubeck2021universal}.
\end{proof}

\begin{lemma}\label{lem:raw-corr-to-centered}
    Denote the events
    \begin{align*}
        E_f = \left\{ \frac{1}{n} \sum_{i = 1} ^n (f(x_i) - \mathbb{E}^{\mu_{\ell_i}}[f])z_i \geq \frac{\epsilon}{8} \right\} \qquad F = \left\{ \frac{1}{n} \sum_{\ell = 1} ^k \left| \sum_{i \in S_\ell} z_i \right| \geq \frac{\epsilon}{8} \right\}.
    \end{align*}
    Then $D_f \implies E_f \cup F$ and $\mathbb{P}(F) \leq 2k \exp\left( -\frac{n\epsilon^2}{8^3k}\right)$. 
\end{lemma}
\begin{proof}
    Follows from Theorem~2 of \cite{bubeck2021universal}.
\end{proof}

From $A_f \cap B \cap C \implies D_f$ and $D_f \implies E_f \cup F$, it follows that
\begin{equation}\label{eq:events}
    A_f \implies (B^c \cup C^c) \cup E_f \cup F.
\end{equation}
In particular $A_{f^W} \implies (B^c \cup C^c) \cup E_{f^W} \cup F$, and using the bounds in Lemmas~\ref{lem:noise-to-raw-corr} and \ref{lem:raw-corr-to-centered} now yields Lemma~\ref{lem:noise-to-corr}. 

\subsubsection{Proof of Lemma~\ref{lem:mi-bound}}

Lemma~\ref{lem:mi-bound} follows from Theorem~3 of \cite{xu2017information}. After translating notation, we have that $\mathbb{P}(|X_w| > \alpha) \leq \delta$ whenever 
\begin{equation}
    n > \frac{a_1^2 C}{\alpha^2}\left(\frac{I(\{X_t\}_{t \in T}; W)}{\delta}+\log \frac{2}{\delta}\right),
\end{equation}
upon which rearranging yields exactly Lemma~\ref{lem:mi-bound}.\footnote{The explicit constant $a_1$ here is not the same as in \cite{xu2017information} as a result of different definitions of ``subgaussian'' - in \cite{xu2017information}, the authors define $\sigma$-subgaussian to mean a random variable $U$ satisfies $\log \mathbb{E}[e^{\lambda (U - \mathbb{E} U)}] \leq \lambda^2 \sigma^2 / 2$. 
The assumed condition $\mathbb{P}(|X| \geq t) \leq 2e^{-t^2/C}$, implies $\log \mathbb{E}[e^{\lambda X}] \leq \lambda^2(9C)/(2)$ (exercises 3.1 d and e of \cite{handel}) so $a_1^2 = 72$.}

\subsubsection{Isoperimetry, Subgaussianity, and Crude Bounds}

Having established Lemma~\ref{lem:noise-to-corr}, to show Theorem~\ref{thm:finite} it remains to upper bound $\mathbb P[E_{f^W}]$. 
We first show that $X_t = \frac{1}{n} \sum_{i = 1} ^n (f^t(x_i) - \mathbb{E}^{\mu_{\ell_i}}[f^t])z_i$ is subgaussian. 
As in \cite{bubeck2021universal}, the isoperimetry assumption implies that $\sqrt{\frac{d}{c}} \frac{f(x_i) - \mathbb{E}^{\mu_{\ell_i}}[f]}{L}$ is $2$-subgaussian. Since $|z_i| \leq 2$, it follows that $\sqrt{\frac{d}{c}} \frac{(f(x_i) - \mathbb{E}^{\mu_{\ell_i}}[f])z_i}{L}$ is $8$-subgaussian, by Proposition~1.2 of \cite{bubeck2021universal}. Hence each $X_t$ is $\left(18 \cdot 8 \frac{cL^2}{dn} \right)$-subgaussian. For convenience we thus define $C_0=\frac{144 cL^2}{dn}$.


Note that the right hand side of \eqref{eq:mi-bound} is a strictly monotone decreasing and hence invertible function of $\delta$, for $\delta \in (0, 1)$.
Thus define
\[
s_{C_0}(\delta)
=
a_1 \sqrt{C_0 \left(\frac{I(\{X_t\}_{t \in T}; W)}{\delta} + \log \frac{2}{\delta} \right)},
\]
where we highlight the explicit dependence on $C_0$. Thus Lemma~\ref{lem:noise-to-corr} yields 
\[
\mathbb{P}(|X_W| > s_{C_0}(\delta)) < \delta.
\]

Recall that we seek an upper bound on\footnote{Here in evaluating $\mathbb{E}^{\mu_{\ell_i}}[f^W]$ we consider $W$ to be fixed, even though $W$ is otherwise random. More precisely, we really mean $\eta(W)$ where $\eta(t)=\mathbb{E}^{\mu_{\ell_i}}[f^t]$ for each fixed $f^t\in\mathcal{F}$.}
\begin{equation*}
    \mathbb{P}(E_{f^W}) = \mathbb{P}\left( \frac{1}{n} \sum_{i = 1} ^n (f^W(x_i) - \mathbb{E}^{\mu_{\ell_i}}[f^W])z_i \geq \frac{\epsilon}{8} \right) \leq \mathbb{P}(|X_W| > \epsilon/8).
\end{equation*} 
Because $s_{C_0}$ is invertible, one has the implicit bound $\mathbb{P}(|X_W| > \epsilon/8) < s_{C_0}^{-1}(\epsilon/8)$. More explicitly, we also have the following crude bound. Setting $\delta = \max\left(\frac{ 2^{7} a_1^2 C_0 I(\{X_t\}_{t \in T}; W)}{\epsilon^2}, 
2e^{-\epsilon^2/(2^7 a_1^2 C_0)}
\right)$, we have
\begin{align*}
    a_1 \sqrt{C_0 \left(\frac{I(\{X_t\}_{t \in T}; W)}{\delta} + \log\frac{2}{\delta}\right)} \geq \frac{\epsilon}{8}
\end{align*}
yielding
\begin{equation}\label{eq:re-mi-bound}
    \mathbb{P}(|X_W| > \epsilon/8) < s^{-1}_{C_0}(\epsilon/8) \leq \max\left(\frac{ 2^{7} a_1^2 C_0 I(\{X_t\}_{t \in T}; W)}{\epsilon^2}, 
    2e^{-\epsilon^2/(2^7 a_1^2 C_0)}
    \right).
\end{equation}
Putting together equations~\eqref{eq:events}~and~\eqref{eq:re-mi-bound} and using that $k\geq 1$, we obtain the following lemma:
\begin{lemma}\label{lem:core-lem}
    \[
        \mathbb{P}(A_{f^W}) \leq 2\exp \left( -\frac{n\epsilon^2}{8^3} \right) + 2k \exp\left( -\frac{n\epsilon^2}{8^3 k}\right) + s^{-1}_{C_0}(\epsilon/8) \leq (2k + 2)\exp\left(-\frac{n\epsilon^2}{8^3 k}\right) + s^{-1}_{C_0}(\epsilon/8).
    \]
\end{lemma}

\subsubsection{Controlling Mutual Information}

We next show Lemma~\ref{lem:mi-control}, which gives a upper bound for $I(\{X_t\}_{t \in T}; W)$.

\begin{proof}[Proof of Lemma~\ref{lem:mi-control}]
First since the chain $[W - \mathcal{D} - \{X_t\}_{t \in T}]$ is Markovian and $f^W = (\mathbf{m}^W, \mathbf{w}^W)$ is a one-to-one function of $W$,
\[
    I(W; \{X_t\}_{t \in T}) \leq I(W; \mathcal{D}) = I(f^W; \mathcal{D}).
\]
Next, by the mutual information chain rule,
one has
\[
    I(f^W; \mathcal{D}) = I\big((\mathbf{m}^W, \mathbf{w}^W); \mathcal{D}\big) 
    = 
    I(\mathbf{m}^W ; \mathcal{D}) + I(\mathbf{w}^W; \mathcal{D} \mid \mathbf{m}^W)
\]
We control the second term via $I(\mathbf{w}^W; \mathcal{D} \mid \mathbf{m}^W) \leq H(\mathbf{w}^W \mid \mathbf{m}^W) \leq \mathbb{E}\left[\log_2 N_{\mathbf{m}^W}\right]$. 
This is because, conditional on the mask $\mathbf{m}^W$, the weights can take on at most $N_{\mathbf{m}^W}$ values (recall the definition $N_{\mathbf{m}} = |\mathcal{W}_{\mathbf{m}}|$.)
\end{proof}
Combining Lemma~\ref{lem:core-lem} with equation~\eqref{eq:re-mi-bound}, followed by applying Lemma~\ref{lem:mi-control} then yields the desired Theorem~\ref{thm:finite}. 

\subsection{The continuous setting}
Our goal will be to show that for our choice of $L$:
\begin{equation}
\label{eq:prob-less-than-delta}
    \mathbb{P}\left( \left\{ \frac{1}{n} \sum_{i = 1} ^n (f^W(x_i) - y_i)^2 \leq \sigma^2 - \epsilon/2 \right\} \cap \left\{ \mathrm{Lip}(f^W) \leq L \right\} \right)
    \leq 
    \delta.
\end{equation}
First, define $W_L$, a modification of $W$, as follows: whenever $\mathrm{Lip}(f^W) > L$, $f^{W_L}$ is instead some prescribed $f^L \in \mathcal{F}$, where $\mathrm{Lip}(f^L) \leq L$. Note that if no such $f_L$ exists, then the probability above is zero regardless and there is nothing to prove. Then the probability in \eqref{eq:prob-less-than-delta} is at most
\begin{equation}\label{eq:tilde-bound}
    \mathbb{P}\left( \left\{ \frac{1}{n} \sum_{i = 1} ^n (f^{W_L}(x_i) - y_i)^2 \leq \sigma^2 - \epsilon/2 \right\} \right).
\end{equation}
where $\mathrm{Lip}(f^{W_L}) \leq L$ always.

Now we proceed with the discretization. For each possible mask $\mathbf{m}$, define 
\[
\mathcal{W}_{\mathbf{m}, L} = \{\mathrm{w} \in \mathcal{W}_{\mathbf{m}}\mid \mathrm{Lip}(f_{\mathbf{m}, \mathbf{w}}) \leq L\}.
\]
Define $\mathcal{W}_{\mathbf{m}, L, \epsilon}$ to be an $\frac{\epsilon}{8J}$-net of $\mathcal{W}_{\mathbf{m}, L}$.  Note that since $\mathrm{diam}(\mathcal{W}) \leq W$ and $\dim \mathcal{W}_{\mathbf{m}, L} \leq \|\mathbf{m}\|_1$, that $N_{\mathbf{m}, L, \epsilon} = |\mathcal{W}_{\mathbf{m}, L, \epsilon}| \leq (1 + 60WJ\epsilon^{-1})^{\|\mathbf{m}\|_1}$ (Corollary~4.2.13 of \cite{VershyninRoman2018Hp:a}). 
We then apply Lemma~\ref{lem:core-lem} to 
\[
\mathcal{F}_{L, \epsilon} \equiv \{f_{\mathbf{m}, \mathbf{w}} \mid \mathrm{Lip}(f_{\mathbf{m}, \mathbf{w}}) \leq L, \mathbf{m} \in \{0, 1\}^p, \mathbf{w} \in \mathcal{W}_{\mathbf{m}, L, \epsilon}\}. 
\]

Define $W_L'$ such that $f^{W_L'}$ uses the same mask as $f^{W_L}$ but the weights of $f^{W_L}$ are rounded to the closest element of $\mathcal{W}_{\mathbf{m}, L, \epsilon}$. 
Then note for two functions $f$ and $g$, one has that if $\|f - g\|_{\infty} \leq \epsilon/8$ and $\|y\|_\infty, \|f\|_\infty, \|g\|_\infty \leq 1$. Since $\|\mathbf{w}^{W_L} - \mathbf{w}^{W_L'}\| \leq \frac{\epsilon}{8J}$, 
$\|f^{W_L} - f^{W_L'}\|_{\infty} \leq \epsilon/8$ by \eqref{eq:param-const}, and thus
\[
    \left|\frac{1}{n} \sum_{i = 1} ^n (y_i - f^{W_L'}(x_i))^2 - \frac{1}{n} \sum_{i = 1} ^n (y_i - f^{W_L}(x_i))^2\right| \leq \epsilon/2,
\]
meaning \eqref{eq:tilde-bound} is again bounded above by the following:
\[
    \mathbb{P}\left(\left\{\frac{1}{n} \sum_{i = 1} ^n (y_i - f^{W'}(x_i))^2 \leq \sigma^2 - \epsilon\right\} \right).
\]
We are now in a position to apply Lemma~\ref{lem:core-lem}, obtaining
\[
    \mathbb{P}\left(\left\{\frac{1}{n} \sum_{i = 1} ^n (y_i - f^{W'}(x_i))^2 \leq \sigma^2 - \epsilon\right\} \right) \leq (2k + 2)\exp\left(-\frac{n\epsilon^2}{8^3 k}\right) + s_{C_0}^{-1}(\epsilon/8)
\]
By our assumption on $k$, the first term is at most $\delta/2$:
\[
    (2k + 2)\exp\left( -\frac{n\epsilon^2}{8^3k}\right) \leq (2k + 2) \frac{\delta}{8k} \leq \frac{\delta}{2}
\]
For the second term, note that $s^{-1}_{C_0}$ is monotone increasing, and thus $s^{-1}_{C_0}(\epsilon/8) \leq \delta/2 \iff \epsilon/8 \geq s_{C_0}(\delta/2)$. Recalling the expression for $s_{C_0}$ and for $C_0$, this condition is simply
\[
    a_1 \sqrt{ \left( 144 \frac{cL^2}{nd} \right) \left( \frac{2I(\{X_t\}_{t \in T}; W_L')}{\delta} + \log \frac{4}{\delta} \right)} \leq \frac{\epsilon}{8}
\]
which is guaranteed when
\[
    L = \frac{\epsilon}{96a_1 \sqrt{2c}} \sqrt{\frac{nd\delta}{I(\{X_t\}_{t \in T}; W_L') + \frac{\delta}{2} \log (4/\delta) }}.
\]
Hence this value of $L$ ensures that
\[
    \mathbb{P}\left( \left\{ \frac{1}{n} \sum_{i = 1} ^n (f^W(x_i) - y_i)^2 \leq \sigma^2 - \epsilon/2 \right\} \cap \left\{ \mathrm{Lip}(f^W) \leq L \right\} \right) \leq \delta.
\]
Then with probability $1 - \delta$ over the data, $f^W$ must have Lipschitz constant $\operatorname{Lip}(f)\geq L$.
Finally, by Lemma~\ref{lem:mi-control}, we have
\[
    I(\{X_t\}_{t \in T}; W_L') \leq I(\mathbf{m}^{W_L'}; \mathcal{D}) + \mathbb{E}[\log_2 N_{\mathbf{m}^W, L, \epsilon}] \leq I(\mathbf{m}^W; \mathcal{D}) + \mathbb{E}[\|\mathbf{m}\|_1] \log_2 (1 + 60WJ\epsilon^{-1}) =: p_{\text{eff}}, 
\]
and hence we conclude that
\[
    \operatorname{Lip}(f) \geq \frac{\epsilon}{96a_1 \sqrt{2c}} \sqrt{\frac{nd\delta}{p_{\text{eff}} + \frac{\delta}{2} \log (4/\delta) }}
\]
with probability $1 - \delta$, concluding the proof of Theorem~\ref{thm:contin}. 

\section{Experimental Details}

\subsection{Experimental Details}\label{app:expdeets}

\subsubsection{Figures \ref{fig:memorization-left} and \ref{fig:memorization-right}}

For Gaussian data, we used $n=30$ data points, each a $d=30$ Gaussian random vector. The same plots persist across $n, d$, these were chosen due to computational resource constraints. The labels are random Boolean in $\{\pm 1\}$. The network is a two-hidden layer MLP with ReLU activation, trained with Cross Entropy Loss and $\eta = 1e-2$ learning rate. We define a train point as ``memorized" if the predictions are within $(\log2)/10$, chosen because predictions being within $\log2$ reflects 50/50 uncertainty, so that a train point is only memorized if the network is highly confidence in its prediction. Our theory is technically for mean-squared error, so we repeated the same experiments with MSE instead and got very similar results, with the same expected patterns holding. Because $n, d$ and the network size are fairly small, we averaged the plots over $k=250$ different networks and datasets with differing random seeds and plotted the average as well as standard error (standard deviation of the mean) over these different trials, finding not much variation so that we may have confidence in our results. 

For FashionMNIST, we added $\sigma^2 = 3$ Gaussian noise to each input image, preserving its original label, allowing us to see if our predictions hold in general settings beyond label corruption and finding, as expected, indeed they do. We train till convergence in loss to within $0.01$ (or until accuracy doesn't change for three consecutive epochs), with $\eta = 1e-3$ on Adam. We use a batch size of $64$ with a two-hidden layer ReLU architecture with a hidden width of $200$. We plotted the mean and standard error (standard deviation of the mean) over $k=10$ different seeds on this dataset. This amount of repetitions with standard datasets like MNIST is common in the literature; \cite{frankle2020pruning} use 5 repetitions on similar datasets like CIFAR10, and less for larger datasets like ImageNet.

We used the pruning code of \cite{tanaka2020pruning} to guarantee consistency in implementation of pruning methods with the rest of the literature. This includes their implementation of all pruning algorithms, though we made modifications to, for instance, add magnitude pruning after training, which they do not include. For the Gaussian case, since the infrastructure for it was not in \cite{tanaka2020pruning}, we wrote our own implementations of the pruning algorithms based on the papers where they were introduced, and checked our implementations' outputs against those of \cite{tanaka2020pruning} where possible, finding agreement. There is further discussion around implementation details of pruning algorithms in \cite{frankle2020pruning}, where they note that their plots mostly agree with those of \cite{tanaka2020pruning}, with some small differences.

\subsubsection{Figures \ref{fig:imp_corr_epochs_small}, \ref{fig:corr-middle}, \ref{fig:corr-right}}

The correlation quantity we track for a fixed dataset is $$\frac{1}{n} \sum_{i=1}^n\left(f\left(x_i\right)-\mathbb{E}^\mu[f]\right) z_i$$ for $n$ the dataset size and $z$ the (fixed) dataset noise, and the data $x_i, y_i$ drawn from our data distribution $\mu$, in our case the output of random Gaussian, with labels arising from the teacher MLP made noisy, with the values of the noise $z_i$ fixed after random generation and used to calculate the correlation. This highlights an important fact, though obvious, that correlation is between the interpolator and the fixed dataset. Training on one dataset will not help you correlate with noise in another, so training is increasing ``expressivity for \textit{this} dataset," in some sense. This quantity lower bounds $I(f^W; \mathcal{D})$, the quantity we wished to track because of its role in $p_{\text{eff}}$, as discussed in the main text. 

In the case of figures \ref{fig:imp_corr_epochs_small}, we track this quantity over the course of IMP on a student-teacher task with noise scale $\sigma^2 =1$. Changing the noise scale simply rescales the y-axis without affecting trend we see. We use a 5-layer ReLU MLP student with hidden width $100$ and a 3-layer teacher with hidden width $50$. This mismatch reflects the fact that our architecture rarely reflects the data-generating process perfectly. We train with Adam with $\eta = 1e-3$, pruning the lowest $20\%$ of weights out at every iteration of IMP, as is standard \cite{frankle2018lottery, frankle2019stabilizing}. We use $n=1000$ data points of dimension $d=50$, where both the student and teacher produce scalar outputs. We average and plot the mean and standard error over different random seeds generating $k=25$ networks (student and teacher) and datasets. The standard error is so small in the plot it is barely visibly. Then for Figure \ref{fig:corr-middle}, \ref{fig:corr-right}, we use the same setup as above except do not prune, and sweep over learning rates and gradient noise instead, tracking the same quantities. These are also averaged over $k=25$ random seeds. 

\subsubsection{Tail end of Figure \ref{fig:imp_corr_epochs_small}}

We see that Figure \ref{fig:imp_corr_epochs_small} gives the expected trend of capacity increasing and decreasing as we train, then prune iteratively during the IMP algorithm to find lottery tickets. Of course, one cannot keep pruning every $100$ epochs forever; eventually one reaches an empty network, so an important sanity check is that such ability to correlate decays at late-time as the network is pruned completely and all weights vanish. We can see below that effective parameter count goes becomes low at extreme sparsities when the network has almost no actual parameters remaining, as one might expect. Again, this is averaged over $k=25$ networks. The critical sparsity level at which this correlation ability begins to decay is the lottery ticket sparsity. For us this occurs after $\approx 9$ pruning iterations where we drop $20\%$ of the weights each time, at which point only $\sim 10\%$ of the weights remain. 

\begin{figure}
    \centering
    \includegraphics[width=200px]{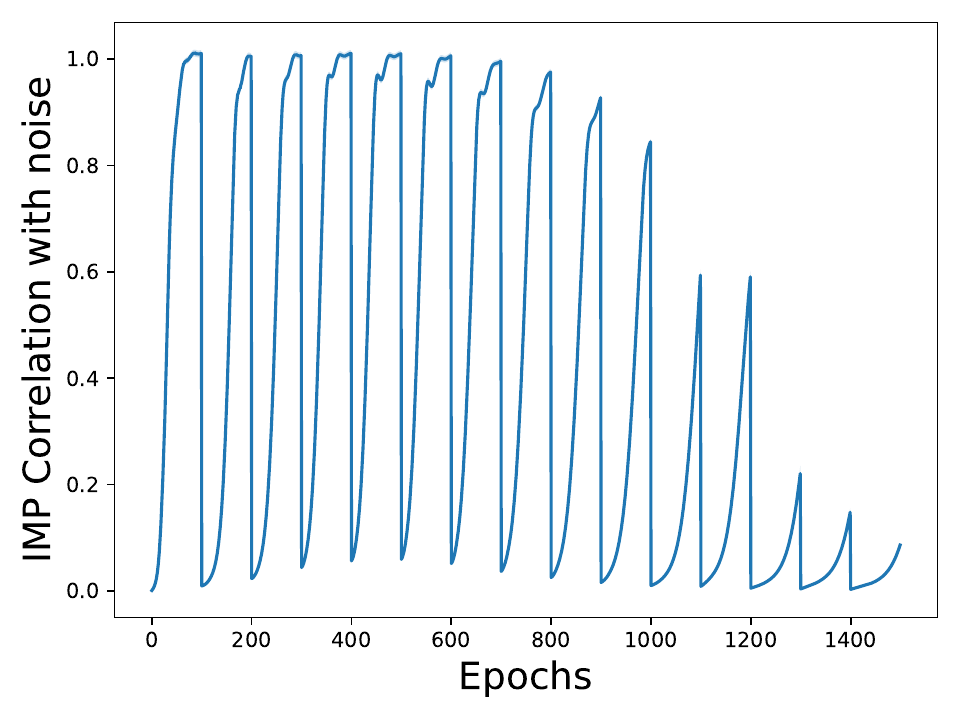} 
    \caption{Correlation with data noise over training epochs, IMP. Expanded version of \ref{fig:imp_corr_epochs_small} pruned to almost complete sparsity.}
    \label{fig:full-wiggly}
\end{figure}

\subsection{Toy model: Computing MI exactly}\label{app:toymodel}

We train very small 1 hidden layer MLPs on small amounts of data. Our toy datasets contain 6 datapoints sampled from $\{\pm1\}^3$ and outputs sampled independently from $\{\pm 1\}$. We train using population gradient descent on MSE
. The MLPs have hidden width $4$ and ReLU activations, with a total of 21 parameters. We measure the mutual information of the mask produced by each pruning method over a range of sparsity levels. SynFlow is omitted here since it is data agnostic and hence its mutual information is identically zero. Following the remarks in the paper, we estimate $I(\mathbf{m}^W; \mathcal{D} \mid \text{init} = \text{init}_0)$; noting that the masking methods are deterministic, this tells us $H(\mathbf{m}^W \mid \mathcal{D}, \mathrm{init}) = 0$, and hence it suffices to estimate $H(\mathbf{m}^W \mid \mathrm{init})$. We do this by directly estimating the probability mass function through sampling. The curves for SNIP and GraSP are jagged due to ReLU causing many of the gradients to be zero, and thus a large portion of the scores are zero to begin. Overall, we see that IMP and magnitude pruning after training have higher mutual information than the other methods. However, these models are so small that at high sparsities, which are the regime of interest, all methods have undergone layer collapse, and thus we are not confident in the representativeness of this experiment. While the maximum size of the sample space is $\binom{21}{10}$, we only sample 32000 times. We note that the graph does not change when moving from $1000$ to $32000$, and thus we believe that the true distribution is supported on far fewer points, and thus our estimates for the mutual information are accurate. We repeat the same experiments for Cross Entropy loss, with hidden layer size 3 (for a total of 21 parameters). 

\begin{figure}[h]
        \centering
        \includegraphics[width=200px]{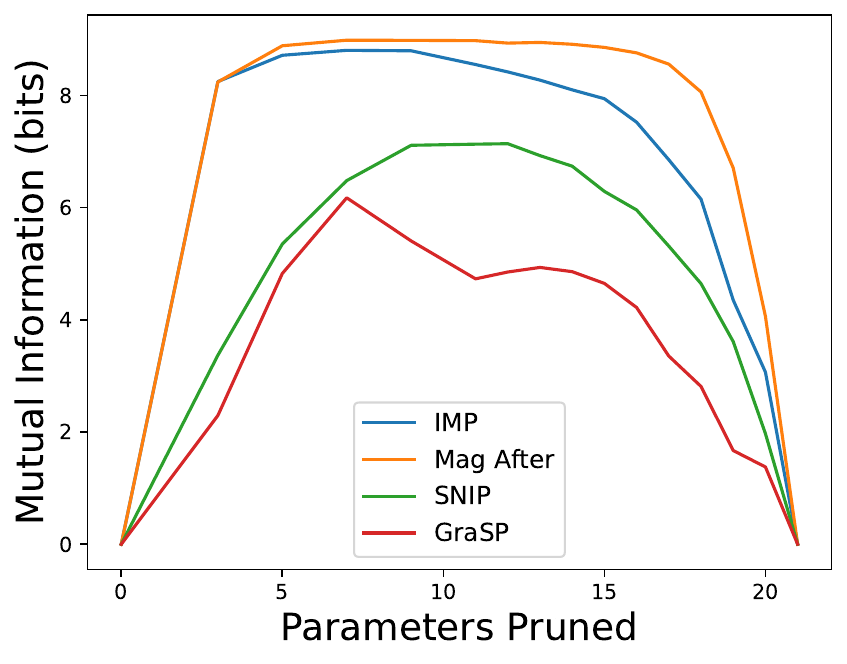}
        \includegraphics[width = 200px]{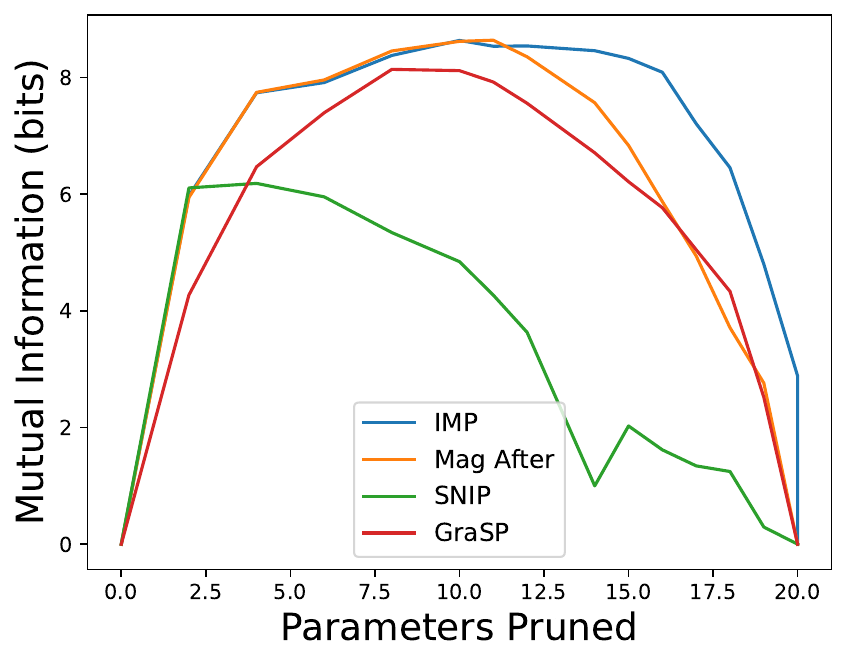}
        \caption{Exact mutual information over on small toy model. Left is regression, right is classification.}
        \label{fig:1a}
\end{figure}

\section{Generalization Bounds}\label{app:gbounds}

\label{app:tradeoff}
\subsection{Understanding the Information-Parameter Tradeoff}

Here, we elaborate on the general principle at the heart of our main result: that for the sake of many theoretical generalization guarantees, parameter count of the model and the model's mutual information with the data are interchangeable. We illustrate this in a toy model. 

\subsubsection{A saturating example}\label{app:tight}
First, we present a toy example which saturates our bound Theorem~\ref{thm:contin}.

We take $n=d^{C}$ for some constant $C$.
For the sake of illustration, consider a dataset of $n$ covariates $x_1,\dots,x_n\stackrel{IID}{\sim}\mathsf{N}(0, I_d/d)$ in dimension $d$, with uniformly random labels $y_1,\dots,y_n \in \{\pm 1\}$. We now consider the following masked function class. First, we sample a single, large matrix with i.i.d. standard Gaussian entries $\mathbf{W} \in \mathbb{R}^{p \times d}$. Next, we construct the following function class:
\begin{equation}
    \mathcal{F} = \{f(x) = (\mathbf{a} \cdot \mathbf{m})^\top \mathrm{ReLU}(\mathbf{W}x + \mathbf{b}) \mid \mathbf{m} \in \{0, 1\}^p \}.
\end{equation}
We will take $\mathbf{b}\in\mathbb R^p$ to have all coordinates $-0.8$. First let us point out that $\mathcal{F}$ is a one-hidden layer neural network with ReLU activations, and width $p$. \textbf{The only learnable parameters of this model are in the choice of mask $\mathbf{m}$ - the weights are randomly drawn at the start and fixed.} Hence this model has $p$ parameters. 
The mask entries determines whether a certain hidden layer node is left nonzero, or is forcibly masked away to zero.

We now consider $p = \exp(O(nd))$. The idea is now that for this toy model, each $x_i$, of which there are $n$, can be memorized by a neuron. Since there are exponentially many hidden nodes, each observing the one dimensional projection $w_j^\top x$, there exists, with high probability, one that aligns very closely with $x_i$ in the sense that $\|w_j-x\|\leq \|x\|/100$.
Call this one $w_{\text{mem}(i)}$, for the memorizer of $i$. Since $n=d^C$, a standard union bound shows that with high probability, for all $i,j$:
\begin{align*}
    \|x_i\|&\in [0.99,1.01],
    \\
    |\langle x_i,x_j\rangle|&\leq d^{-1/3}.
\end{align*}
The triangle inequality then implies (assuming the preceding estimates hold) that $\langle w_{\text{mem}(i)},x_i\rangle\in [0.9,1.1]$ and
\[
    |\langle w_{\text{mem}(i)},x_j\rangle|
    \leq 
    1/10,\quad\forall j\neq i.
\]
Thus, we can set the value of $\mathbf{b}$ to be $-0.8$ in all entries, so that $w^\top _{\text{mem}(i)} x_i + b_{\text{mem}(i)}$ is positive only for the $i$-th covariate, and negative for all the others. After applying ReLU, it is only nonzero for the $i$-th sample -- hence, the $(\text{mem}(i))$-th hidden node indeed memorizes the $i$-th sample.

Now consider the pruning algorithm which looks at the data and sets the mask to only be nonzero on the nodes $\{\text{mem}(i) \mid i \in [n]\}$. 
This function is $O(1)$-Lipschitz on $\{x\in\mathbb R^d~:~\|x\|\leq 1.1\}$.
This is because such any such $x$ satisfies $\langle x,w_{\text{mem}(i)}\rangle\geq 0.8$ for at most $1$ value of $i$. 
This follows easily from the fact that different $w_{\text{mem}(i)}\approx x_i$ are approximately orthogonal, and $0.8^2+0.8^2=1.28$ is significantly larger than $1.1^2$.

Moreover this function has only $s = n \ll nd$ nonzero parameters. However it clearly has $\Theta(nd)$ mutual information with the data ($\Theta(d)$ bits per sample); hence it has $\Theta(nd)$ effective parameters. Thus, this presents a scenario saturating our bound (Theorem~\ref{thm:contin}). In some sense, these are ``lottery tickets" - if we trained this network, we expect a weight $w_k$ with high alignment to a given input to have high learned coefficient $a_k$; our mask simply selectively chooses a single node with high alignment for each sample.

More generally, Theorem~\ref{thm:contin} implies that if one wants $O(nd)$ effective parameters with $O(1)$ Lipschitz constant, while maintaining $s \ll O(nd)$ nonzero parameters, then $\Omega(nd)$ mutual information with the data is required.
This means that, including the mask, one actually needs $\exp(\Omega(nd/s))$ total parameters in the model (c.f. \cite{bubeck2021law}, Theorem 4).

\color{black}
\subsection{Information Theoretic Generalization Bounds}

By applying Lemma~\ref{lem:mi-bound}, one can find analogs of effective parameter count in classical generalization bounds.
Namely, a \emph{low mutual information} choice of \emph{function class with good generalization} will also have good generalization, even if the number of such function classes is very large.
To begin, define the \emph{generalization error} for a function $h$ on a dataset $\mathcal{D} = \{(x_i, y_i)\}_{i = 1} ^n$ of i.i.d input-output pairs in $\mathcal{X} \times \mathcal{Y}$ with respect to a loss $\ell$ as
\[\mathrm{gen}_\mathcal{D}(h) = \left| \mathbb{E}[\ell(h(x), y)] - \frac{1}{n} \sum_{i = 1} ^n \ell(h(x_i), y_i) \right|.\] 

\begin{theorem}[VC Dimension Generalization Bound]\label{thm:vc-mi}
Take $\mathcal{Y} = \{0, 1\}$ and $\ell$ to be 0-1 error.
Let $\{\mathcal{H}_t\}_{t \in T}$ be a collection of binary function classes, where each class $\mathcal{H}_t$ has VC dimension at most $d$. Let $W$ be a random variable taking values on $T$, with $I(W; \mathcal{D}) = I$. Then assuming $n > d + 1$, with probability $1 - \delta$ over the sampling of the data,
\begin{equation}\label{eq:vc-mi}
    \sup_{h \in \mathcal{H}_W} \mathrm{gen}_\mathcal{D}(h)
    \leq \frac{4 + \sqrt{d\log\tfrac{2em}{d}} + \sqrt{\tfrac{4a_1^2}{\delta}(I + \delta\log \tfrac{2}{\delta})}}{\sqrt{2m}}
\end{equation}
\end{theorem}
Recall that $n$-sample VC dimension generalization bounds for a function class of VC dimension $d$ scale roughly as $\tilde{O}\left(\sqrt{\frac{d}{n}}\right)$. The bound above instead scales as $\tilde{O}\left( \sqrt{\frac{d + \delta^{-1} I}{2m} }\right)$. Although we now have worse dependence on $\delta$ (from Lemma~\ref{lem:mi-bound}), if one views $\delta$ as a small constant, we recover the same asymptotics with $d + I$ in place of $d$.
This is an analog of our main result if one interprets $d+I$ as an ``effective VC dimension''.


Similar results also hold for Rademacher complexity. We define the data-dependent Rademacher complexity of a function class $\mathcal{H}$ as
\begin{equation}
    \mathrm{Rad}_{\mathcal{D}}(\mathcal{H}) = \frac{1}{m} \mathbb{E}_{\boldsymbol{\sigma} \sim \{\pm 1\}^m} \left[\sup_{h \in \mathcal{H}} \sum_{i = 1} ^m \sigma_i h(x_i) \right]
\end{equation}
\begin{theorem}\label{thm:rc-mi}
Let $\ell$ take values in $[-b,b]$ and be 1-Lipschitz in its first argument. Let $\{\mathcal{H}_t\}_{t \in T}$ be a collection of functions and $W$ be a random variable taking values on $T$, with $I(W; \mathcal{D}) = I$. Then with probability $1 - \delta$ over the data sampling,
$$\sup_{h \in \mathcal{H}_W} \mathrm{gen}_{\mathcal D}(h) 
    \leq 2\mathrm{Rad}_{\mathcal{D}}(\mathcal{H}_W) + \frac{6}{\sqrt{m}} \sqrt{\frac{a_1^2b}{\delta} \left(I + \frac{\delta}{2} \log \frac{4}{\delta} \right)}$$
\end{theorem}
The above bound again illustrates that incorporating mutual information into the choice of function class yields generalization bounds which decay gracefully with the value of $I$. In the case that each function class has at most $N$ functions, applying Massart's lemma \cite{shalev2014understanding} yields
\begin{equation}
    \mathrm{Rad}_{\mathcal{D}}(\mathcal{H}_W) \leq \frac{b\sqrt{2 \log N}}{\sqrt{m}}.
\end{equation}
Taking $p=\log N$ as the associated parameter count then recovers a generalization bound of the form $\tilde{O}\left(\sqrt{\tfrac{p + \delta^{-1} I}{m}}\right)$.  

\subsection{Proof of Theorem~\ref{thm:vc-mi} (VC Dimension)}


We begin with a standard generalization bound for a fixed function class $\mathcal{H}$ of VC dimension at most $d$:
\[
    \mathbb{E}_{S \sim \mathcal{D}^m} \left[ \sup_{h \in \mathcal{H}} \mathrm{gen}_{\mathcal{D}}(h) \right] 
    \leq 
    \frac{4 + \sqrt{\log(\tau_\mathcal{H}(2m))}}{\sqrt{2m}}
    \leq 
    \frac{4 + \sqrt{d\log (2em/d)}}{\sqrt{2m}}.
\]
Here the first inequality is Theorem 6.11 from \cite{shalev2014understanding} while the second is Sauer's lemma.
Our strategy will be to argue that the bound above holds with high probability for fixed $\mathcal{H}$ by the bounded differences inequality and conclude using Lemma~\ref{lem:mi-bound}.

\begin{proof}[Proof of Theorem~\ref{thm:vc-mi}]
    We define $Y_t = \sup_{h \in \mathcal{H}_t} \mathrm{gen}_{\mathcal{D}}(h)$ and $X_t = Y_t - \mathbb{E}[Y_t]$. 
    Note that $X_t$ changes by at most $1/m$ if a single data-point is modified. Therefore McDiarmid's bounded differences inequality implies
    \begin{equation}
    \mathbb{P}(|X_t| \geq \varepsilon) = \mathbb{P}(|Y_t - \mathbb{E}[Y_t]| \geq \varepsilon) \leq 2e^{-2\epsilon^2m},
    \end{equation}
    i.e. $X_t$ is $(1/2m)$-subgaussian.  Applying Lemma~\ref{lem:mi-bound} now implies that with probability $1 - \delta$, one has
    \[
    X_W \leq \sqrt{\frac{a_1^2}{2m} \left( \frac{I}{\delta} + \log \frac{2}{\delta} \right)}. 
    \]
    Noting $X_W = Y_W - \mathbb{E}[Y_W \mid W]$ yields that with probability $1 - \delta$, 
\[  
    Y_W \leq \frac{4 + \sqrt{\log(\tau_\mathcal{H}(2m))} + \sqrt{\tfrac{2a_1^2}{\delta}(I + \delta\log \tfrac{2}{\delta})}}{\sqrt{2m}}.
\]
Applying Sauer's lemma and recalling the definition of $Y_W$ concludes the proof.
\end{proof}

\subsection{Proof of Theorem~\ref{thm:rc-mi} (Rademacher Complexity)}
We again begin with a standard Rademacher Complexity based generalization bound for a fixed function class $\mathcal{H}$:
\[
    \mathbb{E}\left[\sup_{h \in \mathcal{H}} \mathrm{gen}_{\mathcal{D}}(h)\right] \leq 2\mathbb{E}_{\mathcal{D}} [\mathrm{Rad}_{\mathcal{D}} (\ell \circ \mathcal{H})]
\]
This is Lemma~26.2 of \cite{shalev2014understanding}. Our strategy is equivalent to the above. 
\begin{proof}
We once more define $Y_t = \sup_{h \in \mathcal{H}_t} \mathrm{gen}_{\mathcal{D}}(h)$ and $X_t = Y_t - \mathbb{E}[Y_t]$. Note that $X_t$ changes by at most $2b/m$ when a given datapoint is changed, and thus McDiarmid's bounded differences inequality again implies that $X_t$ is $(2b^2/m)$-subgaussian. Then applying Lemma~\ref{lem:mi-bound} yields that with probability $1 - \delta$,
\[ 
    Y_W \leq \frac{1}{\sqrt{m}} \sqrt{\frac{2a_1^2 b^2}{\delta} \left(I + \delta \log \frac{2}{\delta}\right)} + \mathbb{E}[Y_W \mid W].
\]
Applying Lemma 26.2 \cite{shalev2014understanding} above, we obtain the following almost sure inequality:
\[
    \mathbb{E}[Y_W \mid W] \leq 2\mathbb{E}_{\mathcal{D}} [\mathrm{Rad}_{\mathcal{D}}(\ell \circ \mathcal{H})]
\]

We again note that $\mathrm{Rad}_{\mathcal D}(\ell \circ \mathcal{H}_W)$ changes by at most $2b/m$ when changing an individual datapoint. Hence, once centered, it is again $(2b^2/m)$-subgaussian. Applying Lemma~\ref{lem:mi-bound} again produces that with probability $1 - \delta$,
\[
    \left| \mathrm{Rad}_{\mathcal{D}}(\ell \circ \mathcal{H}_W) - \mathbb{E}_{\mathcal{D}}[\mathrm{Rad}_{\mathcal{D}} (\ell \circ \mathcal{H}_W) \mid W] \right| \leq \sqrt{\frac{2a_1^2 b^2}{m} \left(\frac{I}{\delta} + \log \frac{2}{\delta} \right)}.
\]
Thus with probability $1 - 2\delta$,
\[
    \sup_{h \in \mathcal{H}_W} \mathrm{gen}_{\mathcal D}(h) \leq Y_W \leq 2 \mathrm{Rad}_{\mathcal{D}}(\ell \circ \mathcal{H}_W \circ S) + \frac{1}{\sqrt{m}} \sqrt{\frac{18a_1^2b^2}{\delta} \left( I + \delta \log \frac{2}{\delta} \right)}.
\]
Noting that $\ell$ is Lipschitz in its first argument, we apply Talagrand's contraction lemma (see e.g. \cite{ledoux2013probability,shalev2014understanding}) to conclude the proof of Theorem~\ref{thm:rc-mi}. 
\end{proof}

\end{document}